\documentclass[a4paper,10pt]{article}
\usepackage[dvips]{graphics,graphicx}
\usepackage{amssymb}
\usepackage{amsrefs}
\usepackage{verbatim}
\usepackage{mathtools}
\usepackage{esint}
\usepackage{geometry}
\usepackage{amsthm}
\usepackage{amsmath,color}
\usepackage{ verbatim}
\usepackage{algorithmicx}
\usepackage{algpseudocode}
\usepackage{mathtools}
\usepackage{algorithm}
\usepackage{upref}
\usepackage{enumitem}
\usepackage{placeins}
\usepackage{appendix}
\usepackage{hyperref}
\usepackage{makecell}

\allowdisplaybreaks

\newtheorem{theorem}{Theorem}
\newtheorem{conjecture}{Conjecture}
\newtheorem{lemma}{Lemma}[section]
\newtheorem{proposition}[lemma]{Proposition}
\newtheorem{corollary}[lemma]{Corollary}
\newtheorem{assumption}{Assumption}
\newtheorem{remark}[lemma]{Remark}

\DeclareMathOperator{\var}{Var}

\newlength{\leftstackrelawd}
\newlength{\leftstackrelbwd}
\def\leftstackrel#1#2{\settowidth{\leftstackrelawd}%
{${{}^{#1}}$}\settowidth{\leftstackrelbwd}{$#2$}%
\addtolength{\leftstackrelawd}{-\leftstackrelbwd}%
\leavevmode\ifthenelse{\lengthtest{\leftstackrelawd>0pt}}%
{\kern-.5\leftstackrelawd}{}\mathrel{\mathop{#2}\limits^{#1}}}

\newcommand{\ud}{\,\mathrm{d}}
\newcommand{\pp}{\mathbb{P}}
\newcommand{\E}{\mathbb{E}}

\newcommand{\R}{\mathbb{R}}
\newcommand{\barmu}{\bar{\mu}}

\newcommand{\refr}{\mathrm{refr}}
\newcommand{\tdet}{\tilde{t}}

\newcommand{\lw}[1]{\textcolor{blue}{\textbf{[LW: #1]}}}

\providecommand{\keywords}[1]
{
  \small	
  \textbf{Keywords:} #1
}

\title{Complexity of zigzag sampling algorithm for strongly log-concave distributions}
\author{Jianfeng Lu \thanks{Department of Mathematics, Department of Physics, and Department of Chemistry, Duke University, Durham NC 27708, USA (jianfeng@math.duke.edu)} \and {Lihan Wang \thanks{For correspondence; Department of Mathematics, Duke University, Durham NC 27708 (lihan@math.duke.edu). Current Address: Department of Mathematical Sciences, Carnegie Mellon University, 311 Hamerschlag Drive, Pittsburgh, PA, 15213, USA (lihanw@andrew.cmu.edu)}}}

\date{}
\begin{document}
\maketitle

\begin{abstract}
    We study the computational complexity of zigzag sampling algorithm for strongly log-concave distributions. The zigzag process has the advantage of not requiring time discretization for implementation, and that each proposed bouncing event requires only one evaluation of partial derivative of the potential, while its convergence rate is dimension independent. Using these properties, we prove that the zigzag sampling algorithm achieves $\varepsilon$ error in chi-square divergence with a computational cost equivalent to $O\bigl(\kappa^2 d^\frac{1}{2}(\log\frac{1}{\varepsilon})^{\frac{3}{2}}\bigr)$ gradient evaluations in the regime $\kappa \ll \frac{d}{\log d}$ under a warm start assumption, where $\kappa$ is the condition number and $d$ is the dimension.
    
\end{abstract}


\section{Introduction and Main Results}
Monte Carlo sampling from a high-dimensional probability  distribution is a fundamental problem with applications in various areas including Bayesian statistics, machine learning, and statistical physics. Many sampling algorithms, especially those for continuous state space like $\mathbb{R}^d$, are based on continuous time Markov processes. Examples of these processes include the overdamped Langevin dynamics, whose invariant measure is the target measure, the underdamped Langevin dynamics and Hamiltonian Monte Carlo (HMC) \cite{duane1987hybrid}, both augment the state space with a velocity variable $v$, and have the $x$-marginal distribution of the invariant measure as the target measure. For strongly log-concave distributions, all these processes converge to the equilibrium exponentially fast with rates independent of the dimension, making them suitable for sampling purposes. On the other hand, all of these processes require time discretizations for implementation, which not only induces further numerical errors but requires the time step to be small as well, requiring higher computational complexity if a small bias is desired. To remove such bias due to discretization, the conventional procedure is to introduce the Metropolis-Hastings acceptance-rejection step, but rejections indicate waste of computational resources.

\smallskip

A very different line of sampling algorithms have been recently developed in statistical physics and statistics literature \cite{peters2012rejection}, which are based on piecewise deterministic Markov processes (PDMPs)~\cite{davis1984piecewise}. These processes are non-reversible, which may mix faster than reversible MCMC methods \cites{diaconis2000analysis,turitsyn2011irreversible}.
Examples of such samplers include the randomized Hamiltonian Monte Carlo \cite{bou2017randomized}, the zigzag process \cite{bierkens2019zig}, the bouncy particle sampler \cites{peters2012rejection,bouchard2018bouncy}, and some others \cites{vanetti2017piecewise, michel2014generalized, bierkens2020boomerang}. The zigzag and bouncy particle samplers are appealing for big data applications, as they can be unbiased even if stochastic gradient is used \cites{bouchard2018bouncy, bierkens2019zig}.  These algorithms, as they are still relatively new, have not yet been thoroughly analyzed. In particular, no non-asymptotic computational complexity bounds on these algorithms have been established yet, to the best of our knowledge. Our previous work \cite{lu2020explicit} gives explicit exponential convergence rates for the PDMPs with log-concave potentials, which opens the possibility of deriving such complexity bounds for PDMPs, and provides the foundation of this work.

\subsection{Algorithm and Assumptions}
Let $x$ denote the state variable in $\R^d$ where $d$ is the dimension.  The target distribution we want to sample from is denoted by \begin{equation*}
    \ud \mu(x) = Z^{-1} \exp(-U(x))\ud x,
\end{equation*} 
where $U(x)$ is the potential and $Z=\int_{\R^d} \exp(-U(x))\ud x$ is the normalizing constant. Although the zigzag process can also be applied to sample non log-concave distributions, we will restrict our analysis to strongly log-concave distributions, namely, we make the following assumption throughout:
\begin{assumption}\label{assump:condU}
The potential function $U(x)$ satisfies \begin{equation}\label{eqn:condU}
    m \mathrm{Id} \le \nabla^2 U(x) \le L \mathrm{Id},
\end{equation}
for some $0<m\le 1 \le L$. Moreover, $U(x)$ has a unique minimizer at $x=0$, and $U(0)=0$.
\end{assumption}

For any random variable $X$, we use $\rho(X)$ to denote its law. In this paper, we use chi-square divergence to measure the difference between two probability measures: for probability measures $\rho_1,\rho_2$ that $\rho_1 \ll \rho_2$, it is defined as
\begin{equation*}
\chi^2(\rho_1\,\Vert\, \rho_2) := \int_{\R^d} \Bigl(\frac{\ud \rho_1}{\ud \rho_2}-1\Bigr)^2 \ud \rho_2.
\end{equation*}

The zigzag sampling algorithm is based on a piecewise deterministic Markov process, called zigzag process. Besides the variable $x$, we augment the state space by an auxiliary velocity variable taking value in  $\R^d$. A trajectory of the zigzag process, denoted by $(X_t,V_t)$, can be described as follows. Given some initial $(X_0, V_0)$, the position $X_t$ always evolves according to $\frac{\ud}{\ud t}X_t=V_t$, while the velocity $V_t$ is piecewise constant which only changes when bouncing or refreshing events occur at some random time following Poisson clocks. Bouncing events on the $j$-th direction occur with rate $(V_t^{(j)}\partial_{x_j} U(X_t))_+$, and at such an event the velocity $V_t$ changes by flipping its $j$-th component to $-V_t^{(j)}$. Refreshing events occur with rate $\lambda$ for some fixed $\lambda>0$, when the velocity $V_t$ is completely redrawn from the standard normal $\mathcal{N}(0,\mathrm{Id})$.

\smallskip





It has been established \cites{andrieu2018hypocoercivity,bierkens2019ergodicity, lu2020explicit} that under Assumption \ref{assump:condU}, $\rho(X_t,V_t)$ converges to the invariant measure of the zigzag process, which is a product measure of the target measure in $x$ and the standard Gaussian in $v$: \begin{equation*}
    \ud \barmu(x,v)=\ud \mu(x)\ud \nu(v) \ \mbox{ where }\ud \nu(v)=(2\pi)^{-\frac{d}{2}}\exp\bigl(-\frac{|v|^2}{2}\bigr)\ud v.
\end{equation*}
Our analysis relies on the following more quantitative convergence result for zigzag process proved in \cite{lu2020explicit}, which also specifies the optimal choice of refreshing rate $\lambda$.\footnote{\cite{lu2020explicit} shows exponential convergence for the backward equation. By duality the exponential convergence of the backward equation in $L^2(\bar{\mu})$ is equivalent to the exponential convergence of the forward equation in $\chi^2$ with the same rate.} We would like to comment here that the choice of $\lambda= \sqrt{L}$ is completely technical since it optimizes the theoretical convergence rate (up to a universal constant) of the zigzag process established in \cite{lu2020explicit}. The zigzag process is ergodic even if $\lambda=0$ and in practice the choice $\lambda=0$ is common.
\begin{proposition}\cite{lu2020explicit}*{Theorem 1} \label{prop:choiceT}
Under Assumption \ref{assump:condU}, there exists a universal constant $K$ independent of all parameters, such that for any initial density $\barmu_0$, the zigzag process with friction parameter $\lambda = \sqrt{L}$  satisfies
  \begin{equation}\label{eqn:expcvg}
   \chi^2(\rho(X_T,V_T) \,\Vert\, \barmu) \le K\exp \bigl(-\frac{m}{K\sqrt{L}}T\bigr)\chi^2(\barmu_0 \,\Vert\, \barmu).
  \end{equation}
\end{proposition}

The left-hand side of \eqref{eqn:expcvg} controls desired divergence of $\rho(X)$ with respect to the target measure $\mu$, as we have 
\begin{align*}
    \chi^2(\rho(X_T,V_T)\,\Vert\,\barmu) & = \int_{\R^d\times \R^d} \Bigl( \frac{\ud\rho(X_T,V_T)}{\ud\barmu}\Bigr)^2 \ud \barmu(x, v) -1 \\ 
    & = \int_{\R^d}\Bigl(\frac{\ud \rho(X_T)}{\ud \mu}\Bigr)^2 \Bigl(\int_{\R^d} \Bigl(\frac{\ud \rho(V_T \mid X_T)}{\ud \nu(v)} \Bigr)^2\ud \nu(v) \Bigr)\ud \mu(x)-1 \\ 
    & = \int_{\R^d} \Bigl( \frac{\ud \rho(X_T)}{\ud \mu}\Bigr)^2 \Bigl(1+\chi^2\bigl(\rho(V_T\mid X_T) \,\Vert\, \nu \bigr)\Bigr)\ud \mu(x)-1 \\ & \ge \int_{\R^d} \Bigl(\frac{\ud \rho(X_T)}{\ud \mu} \Bigr)^2 \ud \mu(x)-1=\chi^2(\rho(X_T)\,\Vert\,\mu).
\end{align*}
Moreover, we would take initial condition in the form of
\begin{equation}\label{eqn:ic}
    (X_0,V_0)\sim \barmu_0(x,v)= \mu_0(x)\nu(v),
\end{equation}
which implies that $\chi^2(\barmu_0 \, \Vert\, \barmu) = \chi^2(\mu_0 \, \Vert\, \mu)$. 
Therefore, we get
\begin{equation}\label{eqn:expcvg2}
   \chi^2(\rho(X_T) \,\Vert\, \mu) \le K\exp \bigl(-\frac{m}{K\sqrt{L}}T\bigr)\chi^2(\mu_0 \,\Vert\, \mu),
\end{equation}
which suggests the total time $T$ needed to achieve control of chi-square divergence. 

\smallskip

Of course, in practice, we cannot simulate the zigzag process directly, as simulating the Poisson process associated with the bouncing event would require integrating $(V_t^{(j)} \partial_{x_j} U(X_t))_+$ along the trajectory. 
To turn the zigzag process into an efficient and practical sampling algorithm, the Poisson process for the bouncing events are usually simulated using the Poisson thinning trick (see e.g., discussions in \cite{bierkens2019zig}*{Section 3}). Under Assumption \ref{assump:condU}, we will use the following upper bound estimate for the rate:
\begin{equation}\label{eqn:zzupbd}
    (v_i\partial_{x_i} U(x+vt))_+\le \lvert v_i\partial_{x_i} U(x+vt) \rvert \le \lvert v_i \rvert \lvert \partial_{x_i} U(x+vt)  \rvert \le L \lvert v_i \rvert(\lvert x\rvert +t\lvert v\rvert).
\end{equation}
This upper bound has the advantage of not involving evaluations of $U$ and its partial derivatives, which greatly reduces the computational cost, compared with using numerical quadrature for $d$ Poisson clocks. 
The price to pay is the increased frequency of potential bouncing events, which scales like $O(\sqrt{d})$ since the pessimistic bound for the partial derivative $|\partial_{x_i} U(x)| \le |\nabla U(x)| \le L|x|$ typically sacrifices a factor of $O(\sqrt{d})$ in the first inequality.

\smallskip

Following the above discussions, the zigzag sampling algorithm is described in Algorithm~\ref{alg:zigzag}, where Step~\ref{state:magrate} uses the upper bound estimate in \eqref{eqn:zzupbd}, while Steps~\ref{state:thinning1}--\ref{state:thinning2} correspond to the Poisson thinning step. Note that for each potential bouncing event, the algorithm requires one evaluation of $\partial_{x_i} U$ in Step~\ref{state:thinning1}. In practice, typically accessing the partial derivatives of $U$ is the most time consuming step, therefore, in our complexity analysis, we focus on the number of access to partial derivatives. 

\begin{algorithm}[ht]
\caption{The zigzag sampling algorithm~\label{alg:zigzag}}
\flushleft
\hspace*{\algorithmicindent} \textbf{Input:} Terminal time $T$, initial distribution $\mu_0$.

\begin{algorithmic}[1]
\State Draw $x\sim \mu_0$.
\State Set $t \gets 0$.
\State Set $\refr \gets\textrm{true}$.
\While{$t<T$} 
 \If{$\refr$}
\State Draw $v\sim \mathcal{N}(0,\mathrm{Id})$.
   \State Draw $t_{\refr} \sim \text{Exp}\,(\sqrt{L})$.
   \State $t_{\refr} \gets \min\{t_\refr, T-t\}$.
   \State $\refr \gets \textrm{false}$.
\EndIf
\For{$i=1,\cdots,d$}
   \State \label{state:magrate} Draw $\tau_i$ such that $
   \pp(\tau_i \ge s) =\exp\bigl(-sL \lvert v_i\rvert\lvert x \rvert -\frac{s^2}{2}\lvert v_i\rvert\lvert v\rvert \bigr)$.
\EndFor
\State Pick $j=\arg \displaystyle\min_{i=1,\cdots,d} \tau_i$.
\State $\Lambda_j \gets L \lvert v_j\rvert(\lvert x\rvert +\tau_j\lvert v\rvert)$.
\State $t\gets t+\min\{\tau_j, t_{\refr}\}$.
\State $x\gets x+v\min\{\tau_j, t_{\refr}\}$.
\If{$\tau_j<t_{\refr}$}
\State \label{state:thinning1} $\lambda_j \gets (v_j\partial_{x_j} U(x))_+$.
\State Draw $\alpha\sim \mathrm{Unif}(0,1)$;
\If{$\alpha< \frac{\lambda_j}{\Lambda_j}$} 
 \State  $v_j \gets -v_j$.
\EndIf \label{state:thinning2}
\State  $t_{\refr} \gets t_{\refr}-\tau_j$. 
\Else
\State $\refr\gets \mathrm{true}$.
\EndIf
\EndWhile
\State  \Return $x$.
\end{algorithmic}
\end{algorithm}

\smallskip 

We also need the following assumption for technical purposes, as will be discussed after stating the main results:
\begin{assumption}\label{assump:shortT}
The initial distribution $\mu_0(x)$ satisfies a \emph{warm-start} condition:  \begin{equation}\label{eqn:smallDelta0}
    \chi^2(\mu_0\,\Vert\, \mu) \le \exp\Bigl(\frac{d}{8K\kappa \log d}\Bigr),
\end{equation} where $\kappa:=L/m$ is the condition number, and $K$ is the same universal constant as in  \eqref{eqn:expcvg}. 
Furthermore, the initial distribution is concentrated in the sense of
\begin{equation}\label{eqn:smallx0}
    \eta := \pp_{\mu_0} \Bigl(|x|>\sqrt{\frac{2d}{m}}\Bigr) < \frac{1}{4}.
\end{equation} 
\end{assumption}
\begin{remark}
The concentration condition \eqref{eqn:smallx0} can be easily satisfied. By Gaussian Annulus Theorem, if we pick $\mu_0 =\mathcal{N}(0,\frac{1}{m}\mathrm{Id})$, then $\pp_{\mu_0} \bigl(|x|>\sqrt{\frac{2d}{m}}\bigr) \le 3e^{-cd}$ for some universal constant $c$. The failure probability gets smaller if we take $\mu_0 =\mathcal{N}(0,\frac{1}{L}\mathrm{Id})$ or $\mu_0=\mu$. The warm start condition \eqref{eqn:smallDelta0} is more stringent but can be achieved by first running Langevin Monte Carlo (LMC). We will discuss that after presenting our main result.
\end{remark}

\subsection{Main Results}

\begin{theorem}\label{thm:mainthm} 
Under Assumption \ref{assump:condU}, for any prescribed accuracy $\varepsilon>0$, Algorithm~\ref{alg:zigzag} 
outputs a random variable $X$ such that \begin{equation}\label{eqn:mainresult}
    \chi^2(\rho(X) \,\Vert\, \mu)\le \varepsilon, 
\end{equation}
for terminal time $T$ chosen as 
\begin{equation}\label{eqn:choiceT}
    T = K \Bigl(\frac{\sqrt{L}}{m} \bigl(\log \frac{1}{\varepsilon}+\log \chi^2(\mu_0 \,\Vert\, \mu)+\log K\bigr)\Bigr),
\end{equation} 
where $K$ is the universal constant in \eqref{eqn:expcvg}.

Moreover, if $\varepsilon \ge \exp\bigl(-\frac{d}{8K\kappa \log d}\bigr),$ then, under Assumption \ref{assump:shortT}, with probability $1-\frac{C}{\sqrt{L}T}-C\log^{-\frac{3}{2}}d -\eta$, Algorithm \ref{alg:zigzag} returns an output with a computational cost of
\begin{equation*}
O\Bigl(d^{\frac{3}{2}}\kappa^2 \bigl(\log^\frac{3}{2}\frac{1}{\varepsilon}+\log^\frac{3}{2}\chi^2(\mu_0 \,\Vert\, \mu)\bigr)\Bigr)
\end{equation*} 
evaluations of partial derivatives of $U$, where $\eta$ is defined in \eqref{eqn:smallx0} and $C$ is a universal constant. 
\end{theorem}

\begin{remark}
By repeated trials, the theorem implies that for any $\delta\in (0,\frac{1}{4})$, with probability $1-\delta$, Algorithm \ref{alg:zigzag} returns the desired output with a computational cost of \[O\biggl(d^{\frac{3}{2}}\kappa^2\Bigl(\log^\frac{3}{2}\frac{1}{\varepsilon}+\log^\frac{3}{2}\chi^2(\mu_0\,\Vert\,\mu)\Bigr)\log \frac{1}{\delta}\,\Bigl\lvert\log^{-1}\bigl(\frac{1}{\sqrt{L}T}+\log^{-\frac{3}{2}}d+\eta\bigr)\Bigr\rvert\biggr),\]
that is $\widetilde{O}(d^{\frac{3}{2}}\kappa^{2})$
evaluations of partial derivatives of $U$, where $\widetilde{O}(\cdot)$ hides logarithmic factors. 

With the common computational model that $d$ evaluations of partial derivatives of $U$ is equivalent to one evaluation of $\nabla U$ in complexity, the complexity of zigzag is equivalent to $\widetilde{O}(d^{\frac{1}{2}}\kappa^2)$ evaluations of $\nabla U$. 
\end{remark}

Let us explain the choice of $T$ in \eqref{eqn:choiceT}: 
For the zigzag sampling algorithm to reach the target  $\varepsilon$ accuracy according to \eqref{eqn:expcvg2}, the terminal time $T$ needs to be large enough. Meanwhile, the Assumption~\ref{assump:shortT} guarantees that $T$ is not too large, as otherwise we cannot effectively control the number of bouncing events either due to a very large $V$ drawn from a velocity refreshing event or the trajectory reaching regions with large gradient. These motivate our previous Assumption \ref{assump:shortT} on the initial distribution $\mu_0$, as well as the restriction on $\varepsilon$ that it cannot be too small compared to $d$.
We remark that the assumption on $\varepsilon$ is not prohibitive as we are  interested in high dimensional cases and the error threshold is exponentially small in $d$. 


\smallskip

The warm start condition \eqref{eqn:smallDelta0} can be achieved if we start with a Gaussian distribution in $x$ and run Langevin Monte Carlo \begin{equation}\label{eqn:langevinmc}
    X_{n+1} = X_n-h\nabla U(X_n) + \sqrt{2h}\, \xi_n
\end{equation} where $h$ is the step size, and $\xi_n$ are i.i.d. $\mathcal{N}(0,\mathrm{Id})$ random variables. This leads to the following corollary:
\begin{corollary}\label{cor:lmczzfeasiblest}
Let $d\gg 1$. Suppose the potential $U$ satisfies Assumption \ref{assump:condU} for some $\kappa\ge 1$ such that $\kappa^{\frac{9}{5}} \le \frac{d^\frac{4}{5}}{C\log^3 d}$ for some computable (from \cite{erdogdu2020convergence}) universal constant $C$. Then, for any prescribed accuracy $\varepsilon>0$, if we initialize $X_0 \sim \mathcal{N} (0, \frac{1}{2L}\mathrm{Id})$, the hybrid algorithm by first running LMC \eqref{eqn:langevinmc} for $N= d^{4/5}\kappa^{16/5} $ steps with step size $h=\frac{4}{5}d^{-4/5}\kappa^{-16/5}m^{-1} \log \frac{d}{\kappa}$ and then Algorithm \ref{alg:zigzag} up to time $T=  K \Bigl(\frac{\sqrt{L}}{m} \bigl(\log \frac{1}{\varepsilon}+d^\frac{1}{5}\kappa^\frac{4}{5}\log^2 \frac{d}{\kappa}+\log K\bigr)\Bigr) $ outputs a random variable $X$ such that \begin{equation}\label{eqn:mainresultlmc}
    \chi^2(\rho(X) \,\Vert\, \mu)\le \varepsilon.
\end{equation}
Moreover, if $\varepsilon \ge \exp\bigl(-\frac{d}{8K\kappa \log d}\bigr),$ with probability $1-\frac{C}{\sqrt{L}T}-C\log^{-\frac{3}{2}}d -C \exp(Cd^\frac{1}{5}\kappa^\frac{4}{5}\log^2 \frac{d}{\kappa}-cd)$, Algorithm \ref{alg:zigzag} returns an output with a computational cost of
\begin{equation*}
O\Bigl(d^{\frac{1}{2}}\kappa^2 \log^\frac{3}{2}\frac{1}{\varepsilon}+d^\frac{4}{5}\kappa^\frac{16}{5}\log^3 \frac{d}{\kappa}\Bigr)
\end{equation*} 
evaluations of partial derivatives of $U$.
\end{corollary}\begin{proof}
It is easy to verify that our choice of $N,h$ satisfies $h\le \frac{m}{4L^2}$ and $Nh^2 \le \frac{1}{196c\kappa^2 L^2}$ (where $c$ satisfies \cite{erdogdu2020convergence}*{Lemma 14}). Therefore, we may appeal Lemmas 2, 14, 25, 26 of \cite{erdogdu2020convergence}, so that the random variable $X_N$ produced in \eqref{eqn:langevinmc} satisfies \begin{equation} \label{eqn:lmcinitialization}
    \chi^2(\rho(X_N) \| \mu) \le \exp\Bigl(Cd\exp(-Nhm)+CN h^2 \kappa^2 L^2(d+\log N) \Bigr) = \exp\Bigl(Cd^\frac{1}{5}\kappa^\frac{4}{5}\log^2 \frac{d}{\kappa}\Bigr)
\end{equation} for some universal constant $C$. This, combined with our assumption on $\kappa$, guarantees that \eqref{eqn:smallDelta0} holds with $\rho(X_N)$ playing the role of $\mu_0$. We can also check the validity of \eqref{eqn:smallx0} by \begin{align*}
    \pp_{\rho(X_N)}\bigl(|x|>\sqrt{\frac{2d}{m}}\bigr) \le \Bigl(1+\chi^2(\rho(X_N)\|\mu)\Bigr)^\frac{1}{2} \Bigl(\pp_{\mu}\bigl(|x|>\sqrt{\frac{2d}{m}}\bigr) \Bigr)^\frac{1}{2}\le C \exp\Bigl(Cd^\frac{1}{5}\kappa^\frac{4}{5}\log^2 \frac{d}{\kappa}-cd\Bigr) \ll 1.
\end{align*} Therefore we may apply Theorem \ref{thm:mainthm} with $\mu_0=\rho(X_N)$, and derive that the total computational cost (in terms of number of evaluations of $\nabla U$) equals to \[O\Bigl(N+d^\frac{1}{2}\kappa^2\bigl(\log^\frac{3}{2}\frac{1}{\varepsilon}+ \log^\frac{3}{2}\chi^2(\rho(X_N)\|\mu)\bigr)\Bigr)= O\Bigl(d^\frac{4}{5}\kappa^\frac{16}{5}\log^3 \frac{d}{\kappa}+d^\frac{1}{2}\kappa^2 \log^\frac{3}{2}\frac{1}{\varepsilon} \Bigr).\]
\end{proof}
\smallskip

Theorem~\ref{thm:mainthm} guarantees that the zigzag sampling algorithm (Algorithm \ref{alg:zigzag}) outputs a sample from a distribution with $\chi^2$-divergence at most $\varepsilon$ away from the target density for a computational complexity equivalent to $\widetilde{O}(d^\frac{3}{2}\kappa^2)$ partial derivative evaluations (i.e., amounts to $\widetilde{O}(d^\frac{1}{2}\kappa^2)$ gradient evaluations), in the regime $\max\{\kappa, \log \frac{1}{\varepsilon}\} \ll \frac{d}{\log d}$ with a warm-start condition. Corollary \ref{cor:lmczzfeasiblest} establishes that the hybrid LMC-zigzag algorithm outputs a sample for a computational complexity $\widetilde{O}(d^\frac{4}{5}\kappa^\frac{16}{5})$ gradient evaluations. The initialization using LMC is added only for technical reasons as we currently do not have complexity guarantees otherwise with an explicit initial distribution, nor is it necessary for actual implementations. We would also like to comment that our goal is to obtain the best possible scaling in $d$, and the scaling in $\kappa$ might be possibly improved by a more careful analysis.

\smallskip

Our analysis is based on the quantitative convergence rate of the zigzag process established in \cite{lu2020explicit}, which is  $O(\frac{m}{\sqrt{L}})$ for $m$-convex and $L$-smooth potentials. The rest of our proof is based on estimating $\sup |X_t|$ along a single trajectory of the zigzag process and subsequently turn this into an estimate on the number of potential bouncing events, and hence number of partial derivative evaluations. Our analysis utilizes the two important and desirable features of the zigzag sampling process: \begin{itemize}[wide]
    \item The implementation of the zigzag process does not need time discretization, as the velocity in deterministic portion of the trajectory remains constant, which makes it possible to simulate the exact trajectories of the zigzag process while eliminating an important source of error. This is the reason that the complexity of the zigzag process only has logarithmic dependence on $\frac{1}{\varepsilon}$, without Metropolis acceptance/rejection. 
    \item Moreover, for each potential bouncing event of zigzag, only one evaluation of a \emph{partial derivative} of the potential is required, which is $O(d)$ cheaper than a full gradient evaluation in computational cost for usual model of computation.
\end{itemize} 

\smallskip

We would also remark that we quantify the error of distribution in terms of $\chi^2$-divergence, which provides stronger guarantee than total variation, KL divergence or $2$-Wasserstein distance. While $\chi^2$-divergence is relatively convenient for obtaining convergence rates of continuous processes based on Poincar\`e inequality \cites{cao2019explicit, lu2020explicit}, it does not seem easy to use for analyzing discretization error of SDEs. The work \cite{vempala2019rapid} made assumptions of Poincar\`e inequality for the discrete invariant measure, which is difficult to verify. We are fortunate to avoid such problem for zigzag sampler, thanks to the fact that zigzag does not need time discretization. After the first version of this work appears online, \cite{erdogdu2020convergence} established convergence of LMC in $\chi^2$- and R\'enyi divergence, using the exponential convergence of continuous time overdamped Langevin dynamics in R\'enyi divergence \cites{cao2019exponential,vempala2019rapid}.

\subsection{Previous Works} 

Here we focus on results on non-asymptotic analysis of sampling algorithms, which has been a focused research area in recent years. Many sampling algorithms have been analyzed including algorithms based on overdamped Langevin dynamics \cites{dalalyan2017theoretical,durmus2019high,durmus2019analysis,vempala2019rapid,li2019stochastic,ding2020randomlangevin}, underdamped Langevin dynamics \cites{cheng2018underdamped,dalalyan2020sampling,ma2019there,shen2019randomized,ding2020randomunderdamped,monmarche2020high}, Hamiltonian Monte Carlo \cites{mangoubi2017rapid,lee2018algorithmic,chen2019optimal,mangoubi2018dimensionally,bou2020coupling}, or high order Langevin dynamics \cite{mou2019high}, among others. These methods involve discretization of ODEs or SDEs, which yields an error that scales polynomially with step size. Thus the complexity of these algorithms has polynomial dependence on $\varepsilon^{-1}$, where $\varepsilon$ is the desired accuracy threshold. 

\smallskip

Metropolized variants of sampling algorithms, including Metropolized HMC and Metropolis Adjusted Langevin Algorithm (MALA), have also been studied in \cites{dwivedi2018log,chen2020fast, lee2020logsmooth}, the complexities of which have only logarithmic dependence on $\varepsilon^{-1}$, similar to the zigzag sampling process analyzed here. In \cite{dwivedi2018log} the complexity upper bound for MALA is established as $\widetilde{O}(\kappa d+\kappa^\frac{3}{2}d^\frac{1}{2})$ under warm start condition, and $\widetilde{O}(\kappa d^2+\kappa^\frac{3}{2}d^\frac{3}{2})$ with a feasible start. In \cite{chen2020fast} the complexity upper bound for MALA is improved to $\widetilde{O}(\kappa d+\kappa^\frac{3}{2}d^\frac{1}{2})$ with feasible start (where $\mu_0 = \mathcal{N}(0,\frac{1}{L}\mathrm{Id})$). The work \cite{chen2020fast} also established bounds for Metropolized HMC, which is $\widetilde{O}(\kappa d^\frac{11}{12})$ with warm start (which is in fact more stringent than our Assumption~\ref{assump:shortT}) in the regime $\kappa = O(d^\frac{2}{3})$, and $\widetilde{O}(\kappa^\frac{3}{4}d+\kappa^\frac{7}{4}d^\frac{1}{2})$ with feasible start if the target potential function has a bounded Hessian. The complexity upper bound has been improved in~\cite{lee2020logsmooth} to  $\widetilde{O}(\kappa d)$ for both Metropolized HMC and MALA with a feasible start, based on a refined analysis using concentration of gradient norm. In comparison, our result for zigzag relies on a warm start (which is achievable by LMC), while the complexity upper bound has better dependence in $d$. The issue of feasible start will be further discussed in Section~\ref{sec:disc}.

\smallskip

Regarding asymptotic analysis for the convergence of zigzag process, the ergodicity was first established in \cite{bierkens2019ergodicity}. Exponential convergence of the zigzag process is established in \cites{fontbona2016long, bierkens2017piecewise} using a Lyapunov function argument. A central limit theorem of the zigzag process is established in \cite{bierkens2017limit}, and a large deviation principle is established for the empirical measure in \cite{bierkens2019large}. The spectrum of the zigzag process has been studied in \cites{bierkens2019spectral, guillin2020low}. A dimension independent exponential convergence rate for the zigzag process is established in \cite{andrieu2018hypocoercivity}, using the hypocoercivity framework developed in \cite{dolbeault2015hypocoercivity}. Finally, a more quantitative convergence estimate was established in \cite{lu2020explicit}, for which our analysis of the sampling algorithm is based on. 

\section{Strategy of the Proof}
Since Algorithm \ref{alg:zigzag} always simulates exact trajectories of the zigzag process, we see that \eqref{eqn:mainresult} is guaranteed with the correct choice of $T$. Therefore we only need to estimate the computational complexity. The strategy of the proof is to first give an estimate on $\sup_{t\in [0,T]} U(X_t)$ (Lemma \ref{lem:supuxt}), which directly controls $\sup_{t\in [0,T]} |X_t|$. The upper bound on $|X_t|$ in turn provides us an estimate of upper bound on the number of partial derivative evaluations of $U$. The complexity upper bound we derive holds with high probability, while it does not always hold (for example, the number of proposed bouncing events from the Poisson clock might be atypically high), such events only occur with very small probability, which will be controlled in the proof. 

\smallskip

Let $N+1$ be the total number of velocity refreshments (including the initial refreshment), therefore $N$ is a Poisson random variable such that  \begin{equation}\label{eqn:poissonrv}
    \pp(N=n)=\dfrac{(\sqrt{L}T)^n}{n!}e^{-\sqrt{L}T}.
\end{equation} 
Let $0=T_0<T_1<T_2<\cdots<T_N \leq T < T_{N+1}$ be the refresh times, and $V_{T_k}$ be the velocity variable after refreshment at time $T_k$. For $k=1,\cdots,N$, we use $t_k=T_k-T_{k-1}$ to denote the time duration between refreshments. For convenience, we will also denote $t_{N+1}=T-T_N$.

\smallskip

The first step of the proof is the following lemma which controls $\sup_{t \in [0, T]} U(X_t)$ condition on some high probability events. The proof will be deferred to the appendix. 
\begin{lemma}\label{lem:supuxt}
Under Assumptions \ref{assump:condU} and \ref{assump:shortT}, suppose the following conditions hold:
\begin{subequations}\label{eq:condevents}
\begin{align}
    & \frac{1}{2}\sqrt{L}T \le N \le \frac{3}{2}\sqrt{L}T; \label{eqn:condN} \\
    & \lvert V_{T_k} \cdot \nabla U(X_{T_k}) \rvert \le \Bigl(\frac{d}{\sqrt{L}T}\Bigr)^{1/2} |\nabla U(X_{T_k})|, \quad \forall k = 1, \cdots,N; \label{eqn:condV1}\\
    & \lvert V_{T_k} \rvert \le 2\sqrt{d}, \quad \forall k = 1, \cdots,N \label{eqn:condV2} \\
    & U(X_0)\le \kappa d; \label{eqn:condX0} \\
    & \sum_{k=1}^{N+1} t_k^2\le \dfrac{4T}{\sqrt{L}}. \label{eqn:cltsqt}
\end{align} 
\end{subequations}
Then there exists a universal constant $C$ such that \begin{equation}\label{eqn:lm1det}
    \sup_{t\in [0,T]} U(X_t) \le C \sqrt{L}Td.
\end{equation}
\end{lemma}

The next element in the proof is to control the failure event that \eqref{eq:condevents} does not hold. The control of the first four events are relatively straightforward and will thus be directly carried out in the proof of theorem below; we state the probability for the event \eqref{eqn:cltsqt} to hold as the following lemma, which will also be proved in the appendix. 
\begin{lemma}\label{lem:cltsqt}There exists a universal constant $C$ such that, if $\sqrt{L}T>C$, then with probability $1-\frac{2}{\sqrt{L}T}$, condition \eqref{eqn:cltsqt} holds.
\end{lemma}

The final component of the proof is to turn the estimate for $\sup_{t \in [0, T]} U(X_t)$ to an upper bound for the number of proposed bouncing events. 

\begin{proof}[Proof of Theorem \ref{thm:mainthm}]
Let $p_i$ be the probability that condition $i$ in \eqref{eq:condevents} of Lemma~\ref{lem:supuxt} fails. We start with  condition \eqref{eqn:condN} of Lemma \ref{lem:supuxt}. For Poisson process with $t_i$ as the arrival times, we may estimate the first failure probability (here and for the rest of the proofs $C$ denotes a universal constant that may change from line to line) \begin{equation}
    p_a\le \exp(-\frac{1}{C}\sqrt{L}T) \le \frac{C}{\sqrt{L}T}.
\end{equation}
We now check the conditions \eqref{eqn:condV1} and \eqref{eqn:condV2} of Lemma \ref{lem:supuxt}. By Gaussian Annulus Theorem, for each refreshment, we have \begin{equation} \pp(\lvert V_{T_k}\rvert \ge 2\sqrt{d})\le 3e^{-cd},\end{equation} where $c>0$ is some universal constant. We also require $V_{T_k}$ to satisfy $\lvert V_{T_k}\cdot n(X_{T_k}) \rvert \le \bigl(\frac{d}{\sqrt{L}T}\bigr)^{1/2}$, where $n(X_{T_k})= \frac{\nabla U(X_{T_k})}{|\nabla U(X_{T_k})|}$, which has failure probability \begin{align*}
    \pp (\lvert V\cdot n(X) \rvert \ge \Bigl( \frac{d}{\sqrt{L}T} \Bigr)^{1/2}) & = \dfrac{1}{\sqrt{2\pi}}\int_{\bigl(\frac{d}{\sqrt{L}T}\bigr)^{1/2}}^\infty \exp(-\frac{r^2}{2})\ud r \\ & \le \dfrac{1}{\sqrt{2\pi}}\int_{\bigl(\frac{d}{\sqrt{L}T}\bigr)^{1/2}}^\infty \exp\Bigl(-\frac{r}{2} \bigl(\frac{d}{\sqrt{L}T}\bigr)^{1/2}\Bigr)\ud r \\ & \le \sqrt{\dfrac{2}{\pi}}\Bigl(\frac{\sqrt{L}T}{d}\Bigr)^{1/2} \exp(-\frac{d}{2\sqrt{L}T}).
\end{align*}
 Since we have to draw $V$ for $N$ times, cumulatively this yields a failure probability \begin{equation}
    p_b+p_c\le C\biggl(e^{-cd}+\Bigl(\frac{\sqrt{L}T}{d}\Bigr)^{1/2} \exp\Bigl(-\frac{d}{2\sqrt{L}T}\Bigr)\biggr)\E N.
\end{equation}
Recall the assumption $\varepsilon \ge \exp(-\frac{d}{8K\kappa \log d})$ as well as \eqref{eqn:smallDelta0} (and that $\kappa K\log K \le \frac{d}{4\log d}$), which implies that $\sqrt{L}T \le \frac{d}{2\log d}$ for our choice of $T$ as in \eqref{eqn:choiceT}. Together with condition \eqref{eqn:condN}, we derive (neglecting the obviously smaller term $e^{-cd}$) 
\begin{equation*}
    p_b+p_c\le C\sqrt{L}T \Bigl(\frac{\sqrt{L}T}{d}\Bigr)^{1/2} \exp\Bigl(-\frac{d}{2\sqrt{L}T}\Bigr) \le C\log^{-\frac{3}{2}}d.
\end{equation*}
The failure probability for condition \eqref{eqn:condX0} is straightforward to estimate. Using Assumption \ref{assump:condU}, we have \begin{equation*}
    U(X_0) \le \frac{L}{2}|X_0|^2,
\end{equation*}which indicates \begin{equation*}
    p_d \le  \eta = \pp(|X_0| \ge \sqrt{\frac{2d}{m}}).
\end{equation*}
Finally, $p_e$ is already estimated in Lemma~\ref{lem:cltsqt}, which yields $p_e\le \frac{2}{\sqrt{L}T}$. In summary, the total failure probability of \eqref{eq:condevents} can be bounded as
\begin{equation}\label{eq:failprob}
    p_a + p_b + p_c + p_d + p_e \le \frac{C}{\sqrt{L}T}+C\log^{-\frac{3}{2}} d +\eta.
\end{equation}

We now assume that condition \eqref{eq:condevents} holds. Thus, Lemma \ref{lem:supuxt} together with Assumption~\ref{assump:condU} implies that 
\begin{equation}\label{eqn:lengthx}
    \sup_{t\in[0,T]}\lvert X_t \rvert \le \Bigl(\dfrac{2}{m}\sup_{t\in[0,T]} U(X_t)\Bigr)^{1/2} \le C\Bigl( \frac{\sqrt{L}}{m} Td \Bigr)^{1/2}.
\end{equation}
After each refreshment or bouncing event, Algorithm~\ref{alg:zigzag} runs $d$ independent Poisson clocks $\{\tau_i\}_{i=1,\cdots,d}$ defined in Step~\ref{state:magrate} where, noticing $\sum_i \lvert V_i \rvert \le \sqrt{d}\lvert V \rvert \le 2d$, \begin{equation}
    \pp(\min \tau_i \ge t) \ge \exp\Bigl(-tL\lvert X\rvert\sum_i \lvert V_i \rvert-\frac{t^2}{2}\lvert V\rvert \sum_i \lvert V_i \rvert\Bigr) \ge \exp\bigl(-Cd^\frac{3}{2}(L^\frac{5}{4}m^{-\frac{1}{2}}T^\frac{1}{2}t+t^2)\bigr).
\end{equation} This motivates us to consider the following counting process $\tilde{N}_t$: suppose $\tdet_1,\cdots$ are i.i.d. random variables with $\pp(\tdet_i \ge s)=\exp(-As-Bs^2)$ where $A=Cd^{\frac{3}{2}}L^\frac{5}{4}m^{-\frac{1}{2}}T^\frac{1}{2}$ and $B=Cd^{\frac{3}{2}}$, and let $\tilde{N}_t=\inf_n \{\sum_{i=1}^n \tdet_i>t\}$. By construction, the probability of $N>8AT$ under condition \eqref{eq:condevents} is controlled by $\pp(\tilde{N}_T>8AT)$. 
Therefore, it suffices to estimate  $\pp(\tilde{N}_T>8AT)$. 

We compute the expectation of $\tdet_1$ (here notice $A\gg B\gg 1$): \begin{align*}
    \E \tdet_1 & = \int_0^\infty s(A+2Bs) \exp(-As-Bs^2) \ud s \ge \int_0^\frac{A}{B} s(A+2Bs)\exp(-2As) \ud s  \\ & = \frac{1}{4A}+ \frac{B}{2A^3}-\Bigl(\frac{3A}{2B}+\frac{5}{4A}+\frac{B}{2A^3}\Bigr)e^{-\frac{2A^2}{B}}\ge \frac{1}{4A}. \stepcounter{equation} \tag{\theequation}  \label{eqn:exptdet1}
\end{align*}On the other hand, \[
    \E \tdet_1^2 \le \int_0^\infty s^2(A+2Bs) \exp(-As) \ud s= \frac{2}{A^2}+ \frac{12B}{A^4} \le \frac{33}{16A^2}. \]Therefore we may appeal to Kolmogorov's inequality \cite{durrett2019probability}*{Theorem 2.5.2} (here $S_n$ denotes $\sum_{i=1}^n \tdet_i$): \begin{align*} \pp(\tilde{N}_T > 8AT) & = \pp (S_{8AT}<T) = \pp (S_{8AT}- \E S_{8AT}<T - \E S_{8AT}) \leftstackrel{\eqref{eqn:exptdet1}}{\le} \pp (S_{8AT}- \E S_{8AT}<-T) \\ & \le \frac{1}{T^2}\var S_{8AT} = \frac{8A}{T} \var \tdet_1 \le \frac{16}{AT} \le \frac{C}{\sqrt{L}T}.\end{align*}
To sum up, we have established that with high probability the number of partial derivative evaluations is bounded by  \begin{equation*}O(AT)=O(d^{\frac{3}{2}}L^\frac{5}{4}m^{-\frac{1}{2}}T^\frac{3}{2})=O\Bigl(d^\frac{3}{2}\kappa^2 \bigl(\log^\frac{3}{2}\frac{1}{\varepsilon}+\log^\frac{3}{2} \chi^2(\mu_0\,\Vert\,\mu)\bigr)\Bigr).\qedhere
\end{equation*} 
\end{proof}

\section{Discussion} \label{sec:disc}

We establish non-asymptotic complexity bounds for the zigzag sampling algorithm. While we focus on zigzag sampler in this work, we expect that similar analysis for other PDMPs \cites{bouchard2018bouncy,vanetti2017piecewise, michel2014generalized, bierkens2020boomerang} can be carried out. We leave these for future research.

\smallskip

We admit that our warm-start requirement \eqref{eqn:smallDelta0} may be stringent.
We observe that \eqref{eqn:smallDelta0} implicitly requires the condition number $\kappa$ to be much smaller than $d$, as otherwise, if $\kappa \sim d$, \eqref{eqn:smallDelta0} requires $\chi^2(\mu_0\,\Vert\,\mu) = O(1)$ which is unrealistic. Corollary \ref{cor:lmczzfeasiblest} essentially requires $\kappa \ll d^\frac{4}{9}$ for the analysis to hold. This restriction on condition number is not completely unexpected since the zigzag sampler does perform poorly for highly anisotropic densities (see for example numerical results in \cite{michel2014generalized}). 

A major issue of the warm-start assumption comes from our choice of $\chi^2$ divergence, rather than total variation, $2$-Wasserstein distance, or KL divergence as in previous works for non-asymptotic analysis of sampling algorithms. In particular, if we choose the initial condition \begin{equation}\label{eqn:mu0}
    \ud \mu_0(x)=(\dfrac{L}{2\pi})^\frac{d}{2}\exp(-\frac{L|x|^2}{2}) \ud x ,
\end{equation}as in previous works, then for $U(x)=\frac{m|x|^2}{2}$, we have \begin{align*}
    \chi^2(\mu_0\,\Vert\,\mu) & = Z(\frac{L}{2\pi})^d\int_{\R^d} \exp(-L|x|^2+U(x))\ud x -1\\ & = \kappa^\frac{d}{2}(\frac{L}{2\pi})^\frac{d}{2}\int_{\R^d} \exp\bigl(-(L-\frac{m}{2})|x|^2 \bigr)\ud x -1 = \kappa^\frac{d}{2}(\frac{L}{2L-m})^\frac{d}{2}-1,
\end{align*} which violates \eqref{eqn:smallDelta0}. On the other hand, for the same choice of $\mu_0$, as long as $U$ satisfies Assumption \ref{assump:condU}, one can estimate \begin{align*}
    \mathrm{KL}(\mu_0\,\Vert\,\mu) = (\dfrac{L}{2\pi})^\frac{d}{2} \int_{\R^d} \Bigl(\frac{d}{2}\log \frac{L}{2\pi}+\log Z -\frac{L}{2}|x|^2+U(x)\Bigr) \exp(-\frac{L}{2}|x|^2)\ud x\le \frac{d}{2}\log \kappa.
\end{align*}
This means $\log \mathrm{KL}(\mu_0\,\Vert\,\mu)$, and consequently the logarithm of total variation or $2$-Wasserstein distances are much smaller than any algebraic power of $d$, making it suitable for initialization. We hope the following conjecture is true: \begin{conjecture}Under Assumption \ref{assump:condU}, there exists a universal constant $K$ independent of all parameters, such that for any initial density $\barmu_0$, the zigzag process with friction parameter $\lambda = \sqrt{L}$ satisfies
\begin{equation*}
   \mathrm{KL}(\rho(X_T,V_T) \,\Vert\, \barmu) \le K\exp \bigl(- \frac{m}{K\sqrt{L}}T \bigr)\,
   \mathrm{KL}(\barmu_0 \,\Vert\, \barmu).
\end{equation*} 
\end{conjecture} 
\noindent If this is indeed true, we can establish the convergence in KL divergence of the pure zigzag sampler using a feasible start, without using LMC for initialization.

\smallskip
  
Another interesting open question is whether one can find a tighter upper bound  than Step~\ref{state:magrate} of Algorithm \ref{alg:zigzag} in order to reduce the computational complexity, since it magnifies the proposed bouncing rates by $O(\sqrt{d})$. The following lemma, which might be of independent interest, provides a concentration bound for $|\partial_{x_i} U|$ so that we might be able to give up a small probability to obtain a much sharper bouncing rate control.
\begin{lemma}\label{lem:concdU} Let $U(x)$ satisfy Assumption \ref{assump:condU}, then for any $c>0$,
\begin{equation}
    \pp_{\mu}\Bigl(|\partial_{x_i} U|\ge 2\sqrt{L}+2c\sqrt{L}\log d\Bigr) \le 3d^{-c}.
\end{equation}
\end{lemma}
\noindent The proof of this lemma, deferred to the appendix, is inspired by \cite{lee2020logsmooth}, which uses the following Brascamp-Lieb inequality \cite{brascamp2002extensions}:
\begin{lemma}\label{thm:braslieb}
Let $U(x)$ satisfy Assumption \ref{assump:condU}, then for any $g\in H^1(\mu)$, 
\begin{equation}\label{eqn:braslieb}
    \var_{\mu} g\le \int_{\R^d} \nabla g (\nabla^2 U)^{-1}\nabla g \ud \mu.
\end{equation}
\end{lemma}

With Lemma \ref{lem:concdU}, it might be possible to improve Algorithm \ref{alg:zigzag} while surrendering a small probability by replacing Step ~\ref{state:magrate} with $\pp(\tau_i \ge s)=\exp(-cs\sqrt{L}|v_i|\log d)$ since  $(v_i\partial_{x_i} U(x+vs))_+ \le c\sqrt{L}|v_i|\log d$ with high probability. This motivates the following conjecture:
\begin{conjecture}
Under the Assumption \ref{assump:condU}, for any $\kappa$ and $\log \frac{1}{\varepsilon}$ that are both smaller than some algebraic power of $d$, there exists an algorithm that gives a random variable $X$ such that \begin{equation}
    \chi^2(\rho(X) \,\Vert\, \mu)\le \varepsilon.
\end{equation}
Moreover, with high probability, the algorithm requires $O\Bigl(d\kappa\log d \bigl(\log\frac{1}{\varepsilon}+\log\chi^2(\mu_0\,\Vert\,\mu)\bigr)\Bigr)$ evaluations of partial derivatives of $U$. 
\end{conjecture}
Unfortunately there are several difficulties for proving the conjecture. One is that although $\partial_{x_i}U$ does not exceed $O(\log d)$ with high probability, we are unable to control the partial derivatives for a trajectory of the zigzag process. Another issue is that since some trajectories of the zigzag process may go to regions with partial derivatives exceeding $O(\log d)$, we do not always simulate the exact trajectories, which introduces bias in the sampling.

\bigskip  
\noindent\textbf{Acknowledgment.}
This work is supported in part by National Science Foundation via grants CCF-1910571 and DMS-2012286. We would like to thank Murat Erdogdu for pointing us to their complexity analysis of Langevin Monte Carlo in chi-square divergence \cite{erdogdu2020convergence} to remove the warm start assumptions.

\appendix
\section{Proof of Lemma \ref{lem:supuxt}}
\begin{proof} 
Let $\lambda(t)=V_t\cdot \nabla_x U(X_t)$. If no bouncing happens, then \begin{equation*}
    \dfrac{\ud}{\ud t}\lambda(t)= V_t^\top \nabla_x^2 U(X_t) V_t \le L\lvert V_t \rvert^2.
\end{equation*}
In addition, $\lambda(t)$ decreases when bouncing happens, since there is some positive $V_t^{(i)}\partial_{x_i}U(X_t)$ being changed to $-V_t^{(i)}\partial_{x_i}U(X_t)$ while $X_t$ and other $V_t^{(j)}$'s remain unchanged. Therefore, since $\lvert V_t \rvert$ does not change between refreshments, we have for any $t\in (0,T_{k+1}-T_k)$, \footnote{We remark here that $\lambda(t)$ is not well-defined at the bouncing times. Nevertheless, \eqref{eqn:lambdatk} still makes sense since $\lambda(t)$ decreases at the bouncing events, and since we only use \eqref{eqn:lambdatk} in the time integral sense, this will not cause any problem.} \begin{equation}\label{eqn:lambdatk}
    \lambda(T_k+t) \le \lambda(T_k)+tL \lvert V_{T_k} \rvert^2.
\end{equation}
 Notice for a convex function $U(x)$ that satisfy Assumption \ref{assump:condU}, we have by co-coercivity \begin{equation*}
    \lvert \nabla U(x) \rvert^2 \leq 2L U(x),
\end{equation*}
therefore for any $t\in [0,T_{k+1}-T_k)$, and any $\alpha>0$,
\begin{equation}\label{eqn:ptwiseuxt}\begin{aligned}
    U(X_{T_{k}+t})&=U(X_{T_k}) + \int_0^t \lambda(T_k+\tau)\ud \tau \\ & \le U(X_{T_k})+t\lambda (T_k) + \dfrac{Lt^2}{2}\lvert V_{T_k} \rvert^2 \\ & \leftstackrel{\eqref{eqn:condV1},\eqref{eqn:condV2}}{\le} U(X_{T_k})+t\Bigl(\frac{d}{\sqrt{L}T}\Bigr)^{1/2} \lvert \nabla U(X_{T_k}) \rvert + 2 Lt^2d \\ & \le U(X_{T_k})+t \Bigl( \frac{2d\sqrt{L}}{T}\Bigr)^{1/2} \sqrt{U(X_{T_k})} + 2Lt^2d \\ & \le (1+\alpha) U(X_{T_k}) + d\sqrt{L}t^2(\dfrac{1}{\sqrt{2}T \alpha }+2\sqrt{L}).
\end{aligned} \end{equation}
In particular, \begin{equation*}
     U(X_{T_{k+1}})\le (1+\alpha) U(X_{T_k}) + d\sqrt{L}t_{k+1}^2(\dfrac{1}{\sqrt{2}T \alpha }+2\sqrt{L}).
\end{equation*}
Choosing $\alpha= \frac{1}{\sqrt{L}T}$, we have
\begin{equation*}
     U(X_{T_{k+1}})\le (1+\alpha) U(X_{T_k}) + C dLt_{k+1}^2.
\end{equation*}
Now we apply the above formula iteratively and derive \begin{align*}
    U(X_T) & \le (1+\alpha)^{N+1} U(X_0) +CLd \sum_{k=1}^{N+1}(1+\alpha)^{N-k+1}t_k^2 \\ & \le (1+\alpha)^{N+1}\Bigl(U(X_0)+CLd\sum_{k=1}^{N+1} t_k^2\Bigr) \\ & \leftstackrel{\eqref{eqn:condX0},\eqref{eqn:cltsqt}}{\le}C\sqrt{L}Td. 
\end{align*} Here we used $\alpha =\frac{1}{\sqrt{L}T}=O(\frac{1}{N})$ so $(1+\alpha)^{N+1}=O(1)$, which is true due to \eqref{eqn:condN}, and that $\kappa \le \sqrt{L}T$, which is true with our choice of $T$ in \eqref{eqn:choiceT}.
\end{proof}

\section{Proof of Lemma \ref{lem:cltsqt}}
\begin{proof}
Let $\Xi= \sum_{k=1}^{N+1} t_k^2$. By properties of the Poisson process \cite{durrett1999essentials}, if we condition on $N$,  the distribution of $T_1,T_2,\cdots,T_N$ has the same joint distribution as that of $N$ i.i.d.~random variables uniformly distributed in $(0,T)$. This means
\begin{equation}\label{eqn:poisunif}
    \E (\Xi \mid N) =\dfrac{N!}{T^N}\int_{t_1+\cdots+t_N<T} \Bigl(\sum_{k=1}^N t_k^2 +\bigl(T-\sum_{k=1}^N t_k\bigr)^2\Bigr) \ud t_N\cdots \ud t_1.
\end{equation}
To calculate $\E(\Xi \mid N)$, let us define \begin{equation*}
    I_1(N,T)=\int_{t_1+\cdots+t_N<T}\Bigl(\sum_{k=1}^N t_k^2 +\bigl(T-\sum_{k=1}^N t_k\bigr)^2\Bigr) \ud t_N \cdots \ud t_1
\end{equation*} 
and compute $I_1(N,T)$ by induction in $N$. For $N = 0$, as the sum contains only one term, $I_1(0, T) = T^2$. An easy calculation shows that $I_1(1,T)=\frac{2}{3}T^3$. We will show in general  \begin{equation}\label{eqn:I1result}
I_1(N,T)=\frac{2(N+1)}{(N+2)!}T^{N+2}.
\end{equation} 
Indeed, suppose \eqref{eqn:I1result} holds for $N-1$, we want to prove \eqref{eqn:I1result} for $N$, the starting point of which is the following observation:
\begin{equation*}
     I_1(N,T)  =\int_{t_1+\cdots+t_N<T} t_1^2 \ud t_N\cdots \ud t_1 +\int_0^T I_1( N-1,T-t_1)\ud t_1.
\end{equation*}
The first integral can be treated by integrating the variables one by one, from $t_N$ to $t_{N-1}$ and then $t_{N-2}$, etc. \begin{equation}\label{eqn:I1t12}\begin{aligned}
    \int_{t_1+\cdots+t_N<T} t_1^2 \ud t_N\cdots \ud t_1 & = \int_{t_1+\cdots+t_{N-1}<T} t_1^2(T-t_1-\cdots-t_{N-1}) \ud t_{N-1}\cdots \ud t_1 \\ & = \frac{1}{2}\int_{t_1+\cdots+t_{N-2}<T} t_1^2(T-t_1-\cdots-t_{N-2})^2 \ud t_{N-2}\cdots \ud t_1 \\ & = \cdots \\ &= \dfrac{1}{(N-1)!}\int_0^T t_1^2(T-t_1)^{N-1} \ud t_1 = \dfrac{2}{(N+2)!}T^{N+2}.
\end{aligned}\end{equation}
By the induction assumption \eqref{eqn:I1result} for $N-1$ we have \begin{equation*}
    \int_0^T I_1( N-1,T-t_1)\ud t_1 = \int_0^T \dfrac{2N}{(N+1)!}(T-t_1)^{N+1}\ud t_1 = \dfrac{2N}{(N+2)!}T^{N+2}. 
\end{equation*}
Combining above with \eqref{eqn:I1t12} we finish the proof for $N$. Therefore 
\begin{equation*}
     \E (\Xi \mid N)=\dfrac{N!}{T^N}\dfrac{2(N+1)T^{N+2}}{(N+2)!}=\dfrac{2T^2}{N+2}.
\end{equation*}
The full expectation $\E \Xi$ follows as $N$ is a Poisson random variable 
\begin{align*}
    \E \Xi &= \sum_{n=0}^\infty \E (\Xi \mid N=n) \pp(N=n) \\ &= \sum_{n=0}^\infty \dfrac{2T^2}{n+2}\dfrac{(\sqrt{L}T)^n}{n!}e^{-\sqrt{L}T} \\ & = 2T^2 e^{-\sqrt{L}T}\sum_{n=0}^\infty \Bigl(\dfrac{(\sqrt{L}T)^n}{(n+1)!}-\dfrac{(\sqrt{L}T)^n}{(n+2)!}\Bigr) \\ & = \dfrac{2T}{\sqrt{L}}-\frac{2}{L}+\frac{2e^{-\sqrt{L}T}}{L}\le \frac{2T}{\sqrt{L}}.
\end{align*}

To get the desired estimate, we apply Chebyshev's inequality using the second moment. By the same arguments leading towards \eqref{eqn:poisunif}, we have \begin{equation*}
    \E (\Xi^2 \mid N) =\dfrac{N!}{T^N}\int_{t_1+\cdots+t_N<T}\Bigl(\sum_{k=1}^N t_k^2 +(T-\sum_{k=1}^N t_k)^2\Bigr)^2.
\end{equation*} 
Denote\[I_2(N,T)=\int_{t_1+\cdots+t_N<T}\Bigl(\sum_{k=1}^N t_k^2 +(T-\sum_{k=1}^N t_k)^2\Bigr)^2.\] 
Using the same induction argument as the proof of \eqref{eqn:I1result}, we can prove \begin{equation*}
    I_2(N,T)=\dfrac{4(N+1)(N+6)}{(N+4)!}T^{N+4}. 
\end{equation*}
This can be easily verified for $N=0,1$ and the induction follows form the calculation: 
\begin{align*}
    I_2(N,T) & =\int_{t_1+\cdots+t_N<T} t_1^4 +\int_0^T I_2( N-1,T-t_1)\ud t_1 +2\int_0^T t_1^2 I_1(N-1,T-t_1)\ud t_1 \\ &= \dfrac{1}{(N-1)!}\int_0^T t_1^4(T-t_1)^{N-1} \ud t_1 +\dfrac{4N(N+5)}{(N+3)!} \int_0^T (T-t_1)^{N+3}\ud t_1 \\ & \qquad +\dfrac{4N}{(N+1)!}\int_0^T t_1^2(T-t_1)^{N+1} \ud t_1 \\ &=\dfrac{4(N+1)(N+6)}{(N+4)!}T^{N+4}.
\end{align*}
This shows $\E(\Xi^2\mid N)=\frac{N!}{T^N}I_2(N,T)=\frac{4(N+6)}{(N+2)(N+3)(N+4)}T^4$, and therefore \begin{align*}
    \E \Xi^2 &= \sum_{n=0}^\infty \E (\Xi^2 \mid N=n) \pp(N=n) \\ & = \sum_{n=0}^\infty \dfrac{4T^4(n+6)}{(n+2)(n+3)(n+4)}\dfrac{(\sqrt{L}T)^n}{n!}e^{-\sqrt{L}T} \\ &= 4T^4e^{-\sqrt{L}T} \sum_{n=0}^\infty \Bigl(\dfrac{(\sqrt{L}T)^n}{(n+2)!}-6\dfrac{(\sqrt{L}T)^n}{(n+4)!}\Bigr) \\ & = \frac{4T^2}{L}-\frac{24}{L^2}+8e^{-\sqrt{L}T}(\frac{T^2}{L}+\frac{3T}{L^\frac{3}{2}}+\frac{3}{L^2}).
\end{align*}
This means 
\begin{equation*}
    \E (\Xi-\E\Xi)^2 = \frac{8T}{L^\frac{3}{2}}-\frac{28}{L^2}+8e^{-\sqrt{L}T}(\frac{T^2}{L}+\frac{2T}{L^\frac{3}{2}}+\frac{4}{L^2}-\frac{4e^{-\sqrt{L}T}}{L^2})\le \frac{8T}{L^\frac{3}{2}},
\end{equation*}
where the inequality above holds for $\sqrt{L}T$ larger than some universal constant (which we would assume as it is the interesting parameter regime). 

Finally, to conclude the proof, we apply Chebyshev inequality to  estimate the failure probability as
\begin{equation*}
    \pp (\Xi\ge \frac{4T}{\sqrt{L}}) \le \pp(\Xi-\E \Xi \ge \frac{2T}{\sqrt{L}}) \le \frac{L\E (\Xi-\E\Xi)^2}{4T^2} \le \frac{2}{\sqrt{L}T}. \qedhere
\end{equation*}
\end{proof}
\section{Proof of Lemma \ref{lem:concdU}}
\begin{proof}
The first step is to show that \begin{equation}\label{eqn:avgdU}
    \E_{\mu} |\partial_{x_i}U| \le \sqrt{L}.
\end{equation}
This is straightforward, since using integration by parts,
\begin{equation}
    \E_{\mu} |\partial_{x_i}U|^2=\int_{\R^d} (\partial_{x_i}U)^2 \ud \mu = \int_{\R^d} \partial_{x_i x_i}U \ud \mu \le L,
\end{equation} 
and \eqref{eqn:avgdU} then follows from Cauchy-Schwarz inequality.

The next step is to establish a concentration bound. Let $G(x)=\psi(\partial_{x_i}U)$, where $\psi(a)=\psi(|a|)$ is a smooth nonnegative increasing function satisfying \begin{equation*}
    \psi(0)=\psi'(0)=0, \ \psi(a)=|a| \ \text{ for }|a|\ge 1, \ \text{and }|\psi'(a)|\le 2,
\end{equation*} and $g(x)=\exp(\frac{1}{2}\lambda G(x))$. By the construction of $G$, we have \begin{equation}\label{eqn:emug}
    \E_\mu G=\E_\mu \psi(\partial_{x_i}U) \le 2\E_\mu |\partial_{x_i} U| \le 2\sqrt{L}.
\end{equation}Then $\nabla g(x)=\frac{\lambda}{2}\psi'(\partial_{x_i} U) \nabla (\partial_{x_i} U)  g(x) $. By Lemma \ref{thm:braslieb} for $g(x)$, we have
\begin{align*}
    \E_{\mu} \exp(\lambda G) & -\bigl(\E_\mu \exp\bigl(\dfrac{\lambda G}{2}\bigr)\bigr)^2 = \var_{\mu} g(x) \\
    & \le  \dfrac{\lambda^2}{4}\int_{\R^d}(\psi'(\partial_{x_i} U))^2 \nabla (\partial_{x_i}U) (\nabla^2 U)^{-1}\nabla (\partial_{x_i} U) g^2(x)\ud \mu \\ & \le \lambda^2 \int_{\R^d}  \nabla (\partial_{x_i}U) (\nabla^2 U)^{-1}\nabla (\partial_{x_i} U) g^2(x)\ud \mu  \\ & = \lambda^2 \int_{\R^d} \partial_{x_i x_i} U g^2(x)\ud \mu \le \lambda^2 L\E_\mu \exp(\lambda G).
\end{align*}
Thus for $\lambda \le \frac{1}{2\sqrt{L}}$ we have \begin{equation}\label{eqn:recursive}
    \E_\mu \exp(\lambda G) \le \dfrac{1}{1-\lambda^2 L}\bigl(\E_\mu \exp(\dfrac{\lambda G}{2})\bigr)^2.
\end{equation}
Now we use \eqref{eqn:recursive} recursively, and we obtain for $H(\lambda):=\E_\mu \exp(\lambda G)$,
\begin{equation}\label{eqn:Hlambda}
    H(\lambda)\le \prod_{k=0}^\infty \Bigl(\dfrac{1}{1-\frac{\lambda^2 L}{4^k }}\Bigr)^{2^k}\lim_{\ell\to \infty}H(\dfrac{\lambda}{\ell})^\ell.
\end{equation}
Notice 
\begin{equation}\label{eqn:Hlamel}
    \lim_{\ell\to \infty} H(\dfrac{\lambda}{\ell})^\ell = \lim_{\ell\to \infty}\Bigl(\E_\mu \exp(\dfrac{\lambda G}{\ell}) \Bigr)^\ell =\lim_{\ell\to \infty} \Bigl( 1+\E_\mu\dfrac{\lambda G}{\ell} \Bigr)^\ell = \exp(\lambda\E_\mu  G).
\end{equation} 
Moreover, by \cite{bobkov1997poincare}*{Proposition 4.1}, 
\begin{equation}\label{eqn:prod}
    \prod_{k=0}^\infty \Bigl(\dfrac{1}{1-\frac{\lambda^2 L}{4^k }}\Bigr)^{2^k} \le \frac{1+\lambda \sqrt{L}}{1-\lambda\sqrt{L}}.
\end{equation}Substituting \eqref{eqn:Hlamel} and \eqref{eqn:prod} into \eqref{eqn:Hlambda}, we obtain
\begin{equation*}
    H(\lambda) \le \dfrac{1+\lambda \sqrt{L}}{1-\lambda \sqrt{L}}\exp(\lambda \E_\mu G).
\end{equation*}
Finally, combining the above exponential moment bound of $G$ with Chebyshev inequality, we get 
\begin{equation*}
    \pp_\mu \Bigl(G(x)\ge \E_\mu G+ r \Bigr) \le \exp(-\lambda r)\dfrac{1+\lambda \sqrt{L}}{1-\lambda \sqrt{L}}.
\end{equation*}
Now take $\lambda=1/2\sqrt{L}$, and $r=2c\sqrt{L}\log d$, and using \eqref{eqn:emug} (noticing $G(x)=|\partial_{x_i} U|$ when $G(x)\ge r$ since $r\ge 1$), we arrive at
\begin{equation*}
     \pp_\mu \Bigl(|\partial_{x_i} U|\ge 2\sqrt{L}+ 2c\sqrt{L}\log d \Bigr) \le 3d^{-c}. \qedhere
\end{equation*}
\end{proof}

\bibliographystyle{amsplain}
\bibliography{main}

\end{document}